\documentclass[]{l4dc2020}

\usepackage{wrapfig}
\usepackage{times}

\newcommand{\LTLEVENTUALLY}{\ensuremath{ F}}
\newcommand{\LTLALWAYS}{\ensuremath{ G }}
\newcommand{\tuple}[1]{\ensuremath{\left \langle #1 \right \rangle }}
\newcommand{\customfootnotetext}[2]{{
  \renewcommand{\thefootnote}{#1}
  \footnotetext[0]{#2}}}

\title[STL Learning]{Tractable Reinforcement Learning of Signal Temporal Logic Objectives}

\author{%
 \Name{Harish Venkataraman} \Email{kumaa001@umn.edu}
 \AND
 \Name{Derya Aksaray} \Email{daksaray@umn.edu}
 \AND
 \Name{Peter Seiler} \Email{seile017@umn.edu}\\
 \addr Aerospace Engineering and Mechanics, University of Minnesota, Minneapolis, MN, USA
}

\begin{document}

\maketitle
\begin{abstract}%
 Signal temporal logic (STL) is an expressive language to specify time-bound real-world robotic tasks and safety specifications. Recently, there has been an interest in learning optimal policies to satisfy STL specifications via reinforcement learning (RL). Learning to satisfy STL specifications often needs a sufficient length of state history to compute reward and the next action. The need for history results in exponential state-space growth for the learning problem. Thus the learning problem becomes computationally intractable for most real-world applications. In this paper, we propose a compact means to capture state history in a new augmented state-space representation. An approximation to the objective (maximizing probability of satisfaction) is proposed and solved for in the new augmented state-space. We show the performance bound of the approximate solution and compare it with the solution of an existing technique via simulations.  
\end{abstract}


\section{Introduction}

Reinforcement learning (RL) for controlling unknown or partially known stochastic dynamical systems to satisfy complex time-bound objectives has gained good momentum in the robotics community, e.g., use of an aerial vehicle for infrastructure inspection or environmental monitoring while maintaining a sufficient state of charge throughout the mission. Such complex objectives can be rigorously expressed by temporal logics (TL). 

In the literature, there are numerous works on model-based control synthesis for the satisfaction of TL specifications (e.g., \cite{ding2014,sadigh2014,Fu2014TLRL,lahijanian2015,aksaray2015}). Recently, model-free learning paradigm to satisfy TL specifications has also gained a significant interest (e.g., using RL to find policies that maximize the probability of satisfying a given linear temporal logic (LTL), \cite{Brazdil2014,sadigh2014,Fu2014TLRL,Xiao2017_long,Xiao2019_long,Zhe2019_long}). $Q$-learning, a variant of RL has also been shown to be successful in learning policies satisfying signal temporal logic specifications.[\cite{aksaray2016q}].

Signal temporal logic (STL), is a rich predicate logic that can specify bounds on physical parameters and time intervals [\cite{maler2004}]. For example, an autonomous UAV needs to recharge itself periodically (every 15 minutes in a 2-hour mission) to avoid crashing onto the ground. Such a time-bound objective can be expressed by STL as $G_{[0,105 \  min]}F_{[0,15 \ min]}f(x,y) \in \mathcal{C}$ where $f(x,y) \in \mathcal{C}$ refers to the vehicle position $(x,y)$ being inside a recharging station. 

STL does not have a graph representation such as an automaton to book-keep history. Thus, \cite{aksaray2016q} constructed a higher dimensional Markov Decision Process (MDP) model, known as a $\tau$-MDP, for learning.  This $\tau$-MDP model stores the state history (including the current state) over a finite window of length $\tau$ and enables to compute a reward and action at each time step. For instance, if the specification requires visiting region $B$ after region $A$ within $10$-time steps, then STL satisfaction can be verified with the knowledge of at least $10$-time steps. The number of states in the $\tau$-MDP model grows exponentially with the size of $\tau$. For example, if the original MDP has $m$ states, $\tau$-MDP has $m^{\tau}$ states. This state-space explosion renders learning on $\tau$-MDP model impractical for real-world robotics problems with large state-space and long STL horizons.   

The primary focus of this paper is to provide a new augmented system on which learning to satisfy STL specifications is more computationally tractable and thus could scale to problems with longer STL horizons. The basic idea is that both rewards and actions can be computed without exact state history. The reward and the next action can be computed based on the current state and newly defined notion of flags. The flags capture the historic knowledge of the partial satisfaction for each STL sub-formula constituting the STL specification. The new augmented system is defined as a new MDP known as $F$-MDP, which holds the actual system states and the flag states. 

We formulate a learning problem over $F$-MDP in place of $\tau$-MDP [\cite{aksaray2016q}] and propose a technique to learn the optimal policy maximizing the probability of satisfaction. The proposed technique is shown to have polynomial space complexity as a function of $\tau$. Empirical results also support faster learning due to the compactness of $F$-MDP. The rest of this paper is organized as follows: Sections 2 introduces the key concepts, Section 3 defines the problem formally, Section 4 describes the proposed technique in detail, analyzes the performance and computational complexity of the proposed technique, Section 5 provides simulation results and finally Section 6 concludes with future prospects.

\section{Preliminaries}
\subsection{Signal Temporal Logic (STL)}
In this paper, the desired system behavior is described by an STL fragment with the following \em syntax \em
\begin{equation}
\label{STLsyntax}
\begin{array}{rl}
\Phi &:= F_{[a, b]} \phi | G_{[a, b]} \phi \\
\phi &:= \phi \wedge \phi | \phi \vee \phi | F_{[c, d]} \varphi | G_{[c, d]} \varphi \\
\varphi &:= \psi |\neg \varphi | \varphi \wedge \varphi | \varphi \vee \varphi,
\end{array}
\end{equation}
\noindent where $a,b,c,d \in \mathbb{R}_{\geq 0}$ are finite non-negative time bounds; $\Phi$, $\phi$, and $\varphi$ are STL formulae; $\psi$ is predicate in the form of $f(\mathbf{s}) < d$ where $\mathbf{s}: \mathbb{R}_{\geq 0} \rightarrow \mathbb{R}^n$ is a 
signal, $f: \mathbb{R}^n \rightarrow \mathbb{R}$ is a function, and $d \in \mathbb{R}$ is a constant.    
The Boolean operators $\neg$, $\wedge$, and $\vee$ are negation, conjunction (i.e., {\it and}), and disjunction (i.e., {\it or}), respectively. The temporal operators $F$ and $G$ refer to {\it Finally} (i.e., eventually)  and {\it Globally} (i.e., always), respectively. The reader is refered to \cite{maler2004} for details on STL.

For any signal $\mathbf{s}$, let $s_t$ denote the value of $\mathbf{s}$ at time $t$ and let $(\mathbf{s},t)$ 
be the part of the signal that is a sequence of $s_{t^\prime}$ for $t^\prime \in [t,\infty)$. Accordingly, the \em Boolean semantics \em of STL is recursively defined as follows: 
\begin{equation*}\small
 \begin{array}{lllll}
(\mathbf{s},t) \models (f(\mathbf{s})< d) & \Leftrightarrow & f\left(s_t\right) < d & \Leftrightarrow & ((\mathbf{s},t) \models f\left(s_t\right) < d) = 1,\\
(\mathbf{s},t) \models \neg (f(\mathbf{s})< d) &  \Leftrightarrow & \neg \big ( (\mathbf{s},t) \models (f(\mathbf{s})< d) \big) & \Leftrightarrow & ((\mathbf{s},t) \models f\left(s_t\right) < d) = 0, \\
(\mathbf{s},t) \models \phi_1  \wedge \phi_2  &  \Leftrightarrow & (\mathbf{s},t) \models \phi_1 
\text{ and } (\mathbf{s},t) \models \phi_2 & \Leftrightarrow & (\min\{(\mathbf{s},t) \models \phi_1,(\mathbf{s},t) \models \phi_2\})=1, \\
(\mathbf{s},t) \models \phi_1  \vee \phi_2  & \Leftrightarrow  & (\mathbf{s},t) \models \phi_1 
\text{ or } (\mathbf{s},t) \models \phi_2 & \Leftrightarrow &  (\max\{(\mathbf{s},t) \models \phi_1,(\mathbf{s},t) \models \phi_2\})=1, \\
(\mathbf{s},t) \models  \LTLALWAYS_{[a,b]} \phi  &  \Leftrightarrow  & (\mathbf{s},t') \models 
\phi \quad \forall t' \in [t+a,t+b] & \Leftrightarrow &  (\underset{t' \in
[t+a,t+b]}{\min}(\mathbf{s},t^\prime) \models \phi)=1, \\
(\mathbf{s},t) \models \LTLEVENTUALLY_{[a,b]} \phi &  \Leftrightarrow  &
\exists t'
\in [t+a,t+b] \; \text{ s.t. } (\mathbf{s},t') \models 
\phi & \Leftrightarrow &  (\underset{t' \in
[t+a,t+b]}{\max}(\mathbf{s},t^\prime) \models \phi)=1. 
 \end{array}
\end{equation*}
where $( (\mathbf{s},t) \models f\left(s_t\right) < d) = 1$ implies that $f\left(s_t\right) < d$ is true and $((\mathbf{s},t) \models f\left(s_t\right) < d) = 0$ implies that $f\left(s_t\right) < d$ is false, with $0,1 \in \mathbb{R}$. 
For a signal $(\mathbf{s},0)$, i.e., the whole signal starting from time $0$, satisfying $F_{[a,b]} \phi$ means that ``there exists a time within $[a,b]$ such that $\phi$ will eventually be true'', and satisfying $G_{[a,b]} \phi$ means that ``$\phi$ is true for all times between $[a,b]$''.

As in \cite{dokhanchi2014line}, let $hrz(\phi)$ denote the \emph{horizon} of an STL formula $\phi$, which is the required number of samples to resolve any (future or past) requirements 
of $\phi$. The horizon can be computed recursively as
\vskip-1ex
\begin{equation*}
 \begin{array}{rll}
hrz\left( \psi \right) &=& 0, \\
hrz(\phi)&=&b \quad\;\;\; \ \ \ \text{ if } \phi= G_{[a,b]} \psi \text{ or } F_{[a,b]} \psi, \\
hrz(F_{[a,b]}\phi)=hrz(G_{[a,b]}\phi) &=&b+hrz(\phi), \\ 
hrz(\neg \phi) &= &hrz(\phi),\\
hrz(\phi_1 \wedge \phi_2) = hrz(\phi_1 \vee \phi_2) &=& \max\{hrz(\phi_1),hrz(\phi_2) \}, \\
 \end{array}
\end{equation*} 
where $a,b\in \mathbb{R}_{\geq 0}$, $\psi$ is a predicate, and $\phi,\phi_1,\phi_2$ are STL formulae.

STL formula $\Phi$ as defined in \eqref{STLsyntax} allows at most three layers of nesting. The first layer $\Phi$ constitutes of just $F_{[a,b]}\phi$ or $G_{[a,b]}\phi$. The second layer $\phi$ constitutes of more than one STL fragment $\phi_i$ in conjunction using logical operators and their order of precedence of operation. The third layer $\varphi_i$ in $\phi_i$ allows more than one predicate in conjunction using logical operators and their order of precedence of operation. The index variable $i$ is used to specify each STL fragment in the second layers of $\Phi$. This formulation allows for specifying most complex time-bound objectives or safety specifications, involving asymptotic and/or periodic behavior. Throughout the paper, we call $\Phi$ an STL formula and its constituent $\phi_i$ is called the $i^{th}$ STL sub-formula. 

\begin{example}
Consider the regions A and B illustrated in Fig.~\ref{fig:ex1} and a specification as ``visit regions $A$ and $B$ every 3 minutes along a mission horizon of
10 minutes''. Note that the desired specification can be formulated in STL as 
 \begin{equation}
 \label{navSpec}
\begin{array}{ll}
\Phi = & G_{[0,7]} \phi\\
\phi =  & F_{[0,3]}  (s>5 \wedge s<6)\wedge F_{[0,3]} (s>1 \wedge s<2).
\end{array}
\end{equation}
The horizon of $\Phi$ and $\phi$ are $hrz(\Phi)=10$ and $hrz(\phi)=3$, respectively. Let $\phi = \phi_1 \wedge \phi_2$ be made of $\phi_1 =F_{[0,3]}\varphi_1$, where $\varphi_1 = (s>5) \wedge (s<6)$ and $\phi_2 =F_{[0,3]}\varphi_2$, where $\varphi_2 = (s>1) \wedge (s<2)$. Then satisfying $\Phi$ implies satisfying $\bigwedge\limits_{t \in [0,7]} (F_{[t,t+3]} \varphi_{1} \wedge F_{[t,t+3]} \varphi_{2})$. $\Phi$ is the STL formula with constituent sub-formulae $\phi_1$ and $\phi_2$. Let $\mathbf{s^1}$ and $\mathbf{s^2}$ be two signals as illustrated in Fig.~\ref{fig:ex1}. The signal $\mathbf{s^1}$ satisfies $\Phi$ because $A$ and $B$ are visited within $[t,t+3]$ for every $t \in [0,7]$. However, the signal $\mathbf{s^2}$ violates $\Phi$ because region $B$ is not visited within $[0,3]$. Moreover, the satisfaction of $\mathbf{s}$ with respect to $\Phi$ can be computed via the quantitative semantics as follows:
\begin{equation}
\min_{t \in [0,7]} \min \Big\{ \max_{t^\prime \in [t,t+3]} \{(\mathbf{s},t^\prime) \models \varphi_1\}, \max_{t^\prime \in [t,t+3]} \{ (\mathbf{s},t^\prime) \models \varphi_2 \} \Big\}.
\label{eq:robustness_ex}
\end{equation}  
\begin{figure}[htb!]
 \begin{center}
 \includegraphics[height=0.25\textwidth, trim=0.3cm 0cm 0.2cm 0cm]{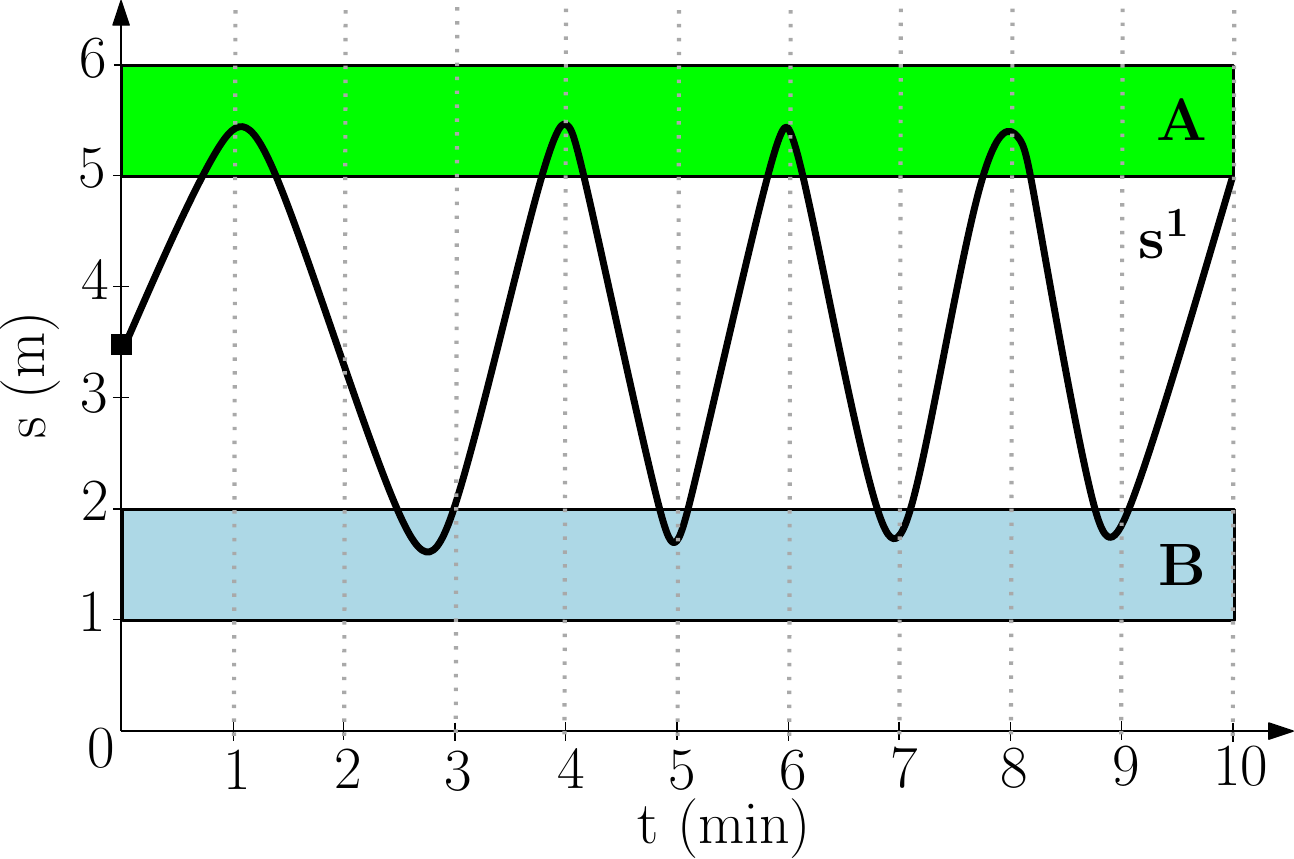} \quad \quad
  \includegraphics[height=0.25\textwidth, trim=0.3cm 0cm 0.2cm 0cm]{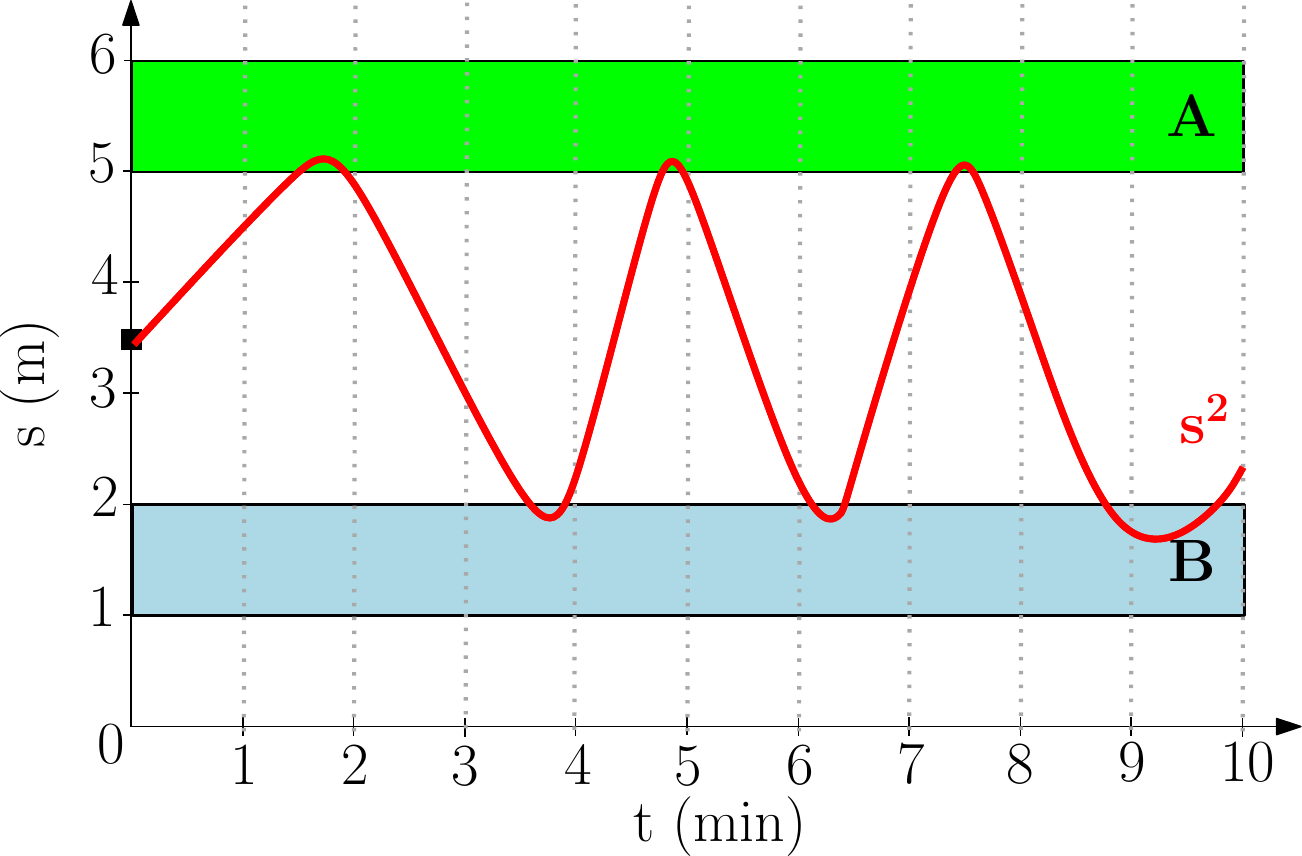}
 \caption{The specification is ``visit regions $A$ and $B$ every 3 minutes along a mission horizon of
10 minutes", i.e., ${\Phi = G_{[0,7]} (F_{[0,3]}  (s>5 \wedge s<6)\wedge F_{[0,3]} (s>1 \wedge s<2))}$, which is satisfied by signal $\mathbf{s^1}$ and violated by signal $\mathbf{s^2}$.}

\label{fig:ex1}
 \end{center}
 \end{figure}
\end{example}

\subsection{Markov Decision Process}
A Markov Decision Process (MDP) is used to model discrete-time, finite state and action space stochastic dynamical systems as ${M=\tuple{\Sigma,A,P,R}}$, where $\Sigma$ denotes the state-space, $A$ denotes the action-space, ${P: \Sigma \times A \times \Sigma \rightarrow [0,1]}$ is the probabilistic transition relation, and $R: \Sigma \rightarrow \mathbb{R}$ is the reward function. If $\sigma \in \Sigma$ is the current state, then the next state $\sigma^\prime \in \Sigma$ on taking an action $a \in A$ is determined using the probabilistic transition function $P$. Given an initial state $\sigma_0$ and action sequence $a_{0:T-1}$, we denote the state sequence generated by $M$ for $T$ time steps as $\sigma_{0:T}$. Every state $\sigma \in \Sigma$ has a scalar measure of value given by the reward function $R$. \\
  
\subsection{Reinforcement Learning: $Q$-learning}
For systems with unknown stochastic dynamics, reinforcement learning can be used to design optimal control policies, i.e., the system learns how to take actions by trial and error interactions with the environment [\cite{sutton1998}].  

$Q$-learning is an off-policy model-free reinforcement learning method [\cite{watkins1992}], which can be used to find the optimal policy for a finite MDP. In particular, the objective of an agent at state $\sigma_{t}$ is to maximize $V(\sigma_{t})$, its expected (discounted) cumulative reward in finite or infinite horizon, i.e., 
\begin{equation}
E\Big[\sum_{k=0}^T R(\sigma_{k+t+1})\Big] \quad \text{ or } \quad E\Big[\sum_{k=0}^\infty \gamma^k R(\sigma_{k+t+1})\Big], 
\label{eq:Qobj}
\end{equation}
where $R(\sigma)$ is the reward obtained at state $\sigma$, and $\gamma$ is the discount factor. Also, ${V^*(\sigma) = \max_a Q^*(\sigma,a)}$, where $Q^*(\sigma,a)$ is the optimal $Q$-function for every state-action pair $(\sigma,a)$. 

Starting from state $\sigma$, the system chooses an action $a$, which takes it to state $\sigma^\prime$ and results in a reward $R(\sigma)$. Then, the $Q$-learning rule is defined as follows:
\begin{equation}
\begin{array}{r}
Q(\sigma,a) := (1-\alpha) Q(\sigma,a) + \alpha [R(\sigma)+\gamma \max\limits_{a^* \in A} Q (\sigma^\prime,a^*)],
\end{array}
\label{eq:Qupdate}
\end{equation}
where $\gamma \in (0,1)$ is the discount factor and $\alpha \in (0,1]$ is the learning rate. Accordingly, if each action $a \in A$ is repetitively implemented in each state $\sigma \in \Sigma$ for infinite number of times and $\alpha$ decays appropriately, then $Q$ converges to $Q^*$ with probability $1$ [\cite{tsitsiklis1994}]. Thus, we can find the optimal policy $\pi^*: \Sigma \rightarrow A$ as $\pi^* = \arg\max_a Q^*(\sigma,a)$.

\section{Problem Statement}

In this paper, we assume that the dynamical system is abstracted as an MDP ${M=\tuple{\Sigma,A,P,R}}$, where $\Sigma$ denotes the set of partitions over a continuous space, $A$ is the set of motion primitives, and each motion primitive $a \in A$ drives the system from the centroid of a partition to the centroid of an adjacent partition. In real-world applications, many systems (e.g., robotic platforms) have uncertainty in their dynamics that are difficult to model (e.g., uncertainty in actuation, gusts in the environment, noises in sensors). In this aspect, we assume that the transition probability function $P$ is unknown in MDP $M$. In other words, given a state-action pair $\sigma_t, a$, the probability distribution of state at next time interval $\sigma_{t+1}$ is unknown. Accordingly, a learning problem can be defined as follows:

\textbf{Problem 1} [Maximizing Probability of Satisfaction] 
\textit{Given an STL specification $\Phi = F_{[0,T]}\phi$ or $G_{[0,T]}\phi$ with a horizon $hrz(\Phi)=T$, a stochastic system model $M=\tuple{\Sigma,A,P,R}$ with unknown $P$ and an initial partial state trajectory $\sigma_{0:\tau}$ for $\tau = \Big\lceil \frac{hrz(\phi)}{\Delta t} \Big\rceil+1$, find a control policy} 
\begin{equation}
\pi^*_1= \arg\max\limits_{\pi} Pr^{\pi} [\sigma_{0:T} \models \Phi] = 
\begin{cases}
\arg\max\limits_{\pi} E^{\pi} \Big[ \max\limits_{t \in [\tau-1,T]} \sigma_{t-\tau+1:t} \models \phi \Big], & \text{if } \Phi=F_{[0,T]} \phi \\
\arg\max\limits_{\pi} E^{\pi} \Big[ \min\limits_{t \in [\tau-1,T]} \sigma_{t-\tau+1:t} \models \phi \Big], & \text{if } \Phi=G_{[0,T]} \phi
\end{cases}
\label{eq:pr1}
\end{equation}
\textit{where $Pr^{\pi} [\sigma_{0:T} \models \Phi]$ is the probability of $\sigma_{0:T}$ satisfying $\Phi$ under policy $\pi$, and $\sigma_{t-\tau+1:t} \models \phi$ is the satisfaction of $\sigma_{t-\tau+1:t}$ with respect to $\phi$.}

 It is shown in \cite{aksaray2016q} that the objective function in \eqref{eq:pr1} is not in the standard form of Q-learning. Hence, they have proposed an approximation of Problem~$1$.

\textbf{Problem 2} [Maximizing Approximate Probability of Satisfaction, \cite{aksaray2016q}]
\textit{Given an STL specification $\Phi = F_{[0,T]}\phi$ or $G_{[0,T]}\phi$ with a horizon $hrz(\Phi)=T$, a stochastic system model $M=\tuple{\Sigma,A,P,R}$ with unknown $P$, known reward function $R$, a log-sum-exp approximation constant $\beta>0$ and an initial partial state trajectory $\sigma_{0:\tau}$ for $\tau = \Big\lceil \frac{hrz(\phi)}{\Delta t} \Big\rceil+1$, find a control policy}
\begin{equation}\label{pr:1a}
\pi^*_{2}= 
\begin{cases}
\arg\max\limits_{\pi} E^{\pi} \Big[ \sum\limits_{t=\tau-1}^T e^{\beta(\sigma_{t-\tau+1:t} \models \phi)} \Big], & \text{if } \Phi=F_{[0,T]} \phi \\
\arg\max\limits_{\pi} E^{\pi} \Big[ - \sum\limits_{t=\tau-1}^T e^{- \beta(\sigma_{t-\tau+1:t} \models \phi)} \Big], & \text{if } \Phi=G_{[0,T]} \phi
\end{cases}
\end{equation}

A new system model is defined using a $\tau$-MDP ${M^\tau=\tuple{\Sigma^\tau,A,P^\tau,R^\tau}}$, with states $\sigma^\tau_t = \sigma_{t-\tau+1:t}$ and solved using $Q$-learning in \cite{aksaray2016q}. This $\tau$-MDP representation suffers from exponential state-space growth with horizon $\tau$ (Curse of history). Hence, we intend to solve the approximate objective \eqref{pr:1a}, in a new system representation without the curse of history.

\section{Proposed Technique}
The proposed technique is based on the definition of a new compact system representation, which is rich enough to capture the necessary history within a finite window horizon. The new representation is sufficient to compute reward and action at each time step.

Let $\Phi$ be $G_{[0,T]} \phi$ or $F_{[0,T]} \phi$, where $\Phi$ is the STL formula with the syntax as given in \eqref{STLsyntax}. Suppose $\Phi$ has $n$ STL sub-formulae $\phi_i$ as $G_{[0,t]} \varphi$ or $F_{[0,t]} \varphi$ with horizon $hrz(\phi_i)=\tau_i, \forall i\in [1,n]$. Then, we assign one discrete valued flag variable ($f_i \in \mathfrak{F_i} := \{k/(\tau_i-1), k \in [0,\tau_1-1]\}$) for each $\phi_i$ and proceed with defining a flag state augmented MDP know as F-MDP, which captures current state and flags to test for satisfaction of each STL sub-formula $\phi_i$. $\mathfrak{F_i}$ is the flag state set of the flag $f_i,i\in [1,n]$ 
 
\begin{definition} [$F$-MDP]
Given MDP $M=(\Sigma,A,P,R)$ and flag state sets $\mathfrak{F_i}, \forall i \in [1,n]$, an $F$-MDP is a tuple ${M^{F}=  
(\Sigma^{F},A,P^{F},R^{F})}$, where 

\begin{itemize}
\item $\Sigma^{F} \subseteq (\Sigma \times{\displaystyle \prod_{i=1}^{n} \mathfrak{F_{i}}}$) is the set of finite states, obtained by the cartesian product between state set and all n flag state sets.
Each state $\sigma^F \in \Sigma^F$ holds the current $\sigma \in \Sigma$ and $f_i \in \mathfrak{F_i}, \forall i \in [1,n]$.
\vskip1ex
\item $P^F: \Sigma^F \times A \times \Sigma^F 
\rightarrow [0,1]$ is a probabilistic transition relation. Let $\sigma^F = \sigma,f_1,f_2,..,f_n$ and $\sigma^{F^\prime} =  \sigma^\prime,f^\prime_1,f^\prime_2,..,f^\prime_n$. $P^F (\sigma^{F}, a ,\sigma^{F^\prime})>0 $ if and only if $P(\sigma, a, \sigma^\prime)>0$ and $f^\prime_i = update(f_i,\sigma), \forall i \in [1,n]$. where $update(.)$ is the flag update rule:

\begin{equation}
  f^\prime_i =
  \begin{cases}
    1, & \text{if $\phi_i = F_{[0,t]}\varphi_i$ and  $\sigma^\prime \vDash \varphi_i $} \\
    min(f_i - 1/(\tau_i-1),0), & \text{if $\phi_i = F_{[0,t]}\varphi_i$ and  $\sigma^\prime \nvDash \varphi_i $} \\
    max(f_i + 1/(\tau_i-1),1), & \text{if $\phi_i = G_{[0,t]}\varphi_i$ and  $\sigma^\prime \vDash \varphi_i $} \\
    0, & \text{if $\phi_i = G_{[0,t]}\varphi_i$ and  $\sigma^\prime \nvDash \varphi_i $} \\
  \end{cases} \ \ \ \ \ \ \forall i \in [1,n].
  \label{rule:flagupdate}
\end{equation}

\vskip1ex
\item $R^F: \Sigma^F \rightarrow \mathbb{R}$ is a reward function. 
\end{itemize}
\label{def:E-MDP}
\end{definition}

From the update rule \eqref{rule:flagupdate}, we can see how each flag $f_i$ can only take on discrete values between $0$ and $1$ with a step size of $1/(\tau_i-1)$. Number of states the flag $f_i$ can take is $\tau_i$ and thus $\mathfrak{F_i}$, has a size equal to horizon $\tau_i$ of the STL sub-formula $\phi_i, \forall i \in [1,n]$.

Given a $\sigma^F = \sigma,f_1,f_2..f_n$, the satisfaction function $sat(\sigma^F,\phi)$ used to test for satisfaction of STL formula $\phi$ and its constituent STL sub-formulae $\phi_i$, is recursively defined as follows:

\begin{equation}
  sat(\sigma^F,\phi_i)  =
  \begin{cases}
    1, & \text{if $ f_i > 0 \ $or$ \ \sigma \vDash \varphi_i$ for $\phi_i = F_{[0,t]}\varphi_i$} \\
    0, & \text{if $ f_i = 0 \ $and$ \ \sigma \nvDash \varphi_i$ for $\phi_i = F_{[0,t]}\varphi_i$}\\
    1, & \text{if $ f_i = 1 \ $and$ \ \sigma \vDash \varphi_i$ for $\phi_i = G_{[0,t]}\varphi_i$}\\
    0, & \text{if $ f_i < 1 \ $or$ \ \sigma \nvDash \varphi_i$ for $\phi_i = G_{[0,t]}\varphi_i$} \\
  \end{cases} \ \ \ \ \ \ \forall i \in [1,n].
  \label{rule:flagupdate}
\end{equation}

\begin{align}\label{eq:sat}
  & sat(\sigma^F,\phi_j \wedge \phi_k) = min(sat(\sigma^F,\phi_j),sat(\sigma^F,\phi_k)), \\
  & sat(\sigma^F,\phi_j \vee \phi_k) = max(sat(\sigma^F,\phi_j),sat(\sigma^F,\phi_k)),
\end{align}

where $\phi_j,\phi_k$ can be STL sub-formulae or their conjunction using logical operators $\wedge,\vee$. For example, if $\phi = ((\phi_1 \wedge \phi_2) \vee \phi_3)$, first current state $\sigma^F$ is evaluated with respect to sub-formulae $\phi_i, \forall i \in [1,3]$. Then $\sigma^F$ is evaluated with respect to $\phi_j \wedge \phi_k, j = 1,k = 2$. Finally $\sigma^F$ is evaluated with respect to new $\phi_j \vee \phi_k$, where $\phi_j = \phi_1 \wedge \phi_2$ and $k=3$.  

The reward function $R^F$ of the problem in the new MDP $M^F$ is as follows:

\begin{equation}
  r =
  \begin{cases}
    e^{\beta sat(\sigma^F,\phi)}, & \text{if $\Phi = F_{[0,T]}\phi $} \\
    -e^{-\beta sat(\sigma^F,\phi)}, & \text{if $\Phi = G_{[0,T]}\phi $} \\
  \end{cases}
\end{equation}
where $\beta>0$ is the log-sum-exp approximation constant. 

The overview of the complete technique to solve \eqref{pr:1a} using F-MDP is as follows:\\
 1) For any STL formula $\Phi$ in accordance with \eqref{STLsyntax} (i.e., $G_{[0,T]} \phi$ or  $F_{[0,T]} \phi$), create one flag per STL sub-formula $\phi_i$ in $\Phi$ and redefine the learning problem in a new flag state augmented state-space $\Sigma^F$ which has new state dimensions corresponding to the flags.\\
 2) Define the objective function such that the agent observes an immediate reward as a function of current state in $\Sigma^F$. After executing these steps, one can use standard $Q$-learning algorithm to find the optimal policy $\pi^*: \Sigma^F \rightarrow A$ in the new $F$-MDP state-space. Overall, we aim to solve the following problem.
 
\textbf{Problem 3} [Maximizing Approximate Probability of Satisfaction with $F$-MDP]
\textit{Let $\Phi$ be STL formula with the syntax in \eqref{STLsyntax}, made up of STL sub-formulae $\phi_i,\forall i \in [1,n]$. Let $T=hrz(\Phi)$, $\tau_i= \Big\lceil \frac{hrz(\phi_i)}{\Delta t} \Big\rceil+1,\forall i \in [1,n]$ and $\tau = \max\limits_{i \in [1,n]}(\tau_i)$. Given an unknown MDP $M$, and flag state sets $\mathfrak{F_i},\forall i \in [1,n]$, $F$-MDP ${M^F=\tuple{\Sigma^F,A,P^F,R^F}}$ can be constructed. Assume that initial $\tau$-states $\sigma_{0:\tau-1}$ are given from which $\sigma^F_{\tau-1}$ can be obtained. Let $\beta>0$ be a known approximation parameter. Find a control policy $\pi^*_{3}: \Sigma^F \rightarrow A$ such that }
\begin{equation}\label{pr:1a}
\pi^*_{3}= 
\begin{cases}
\arg\max\limits_{\pi} E^{\pi} \Big[ \sum\limits_{t=\tau-1}^T e^{\beta sat(\sigma^F_t,\phi)} \Big], & \text{if } \Phi=F_{[0,T]} \phi \\
\arg\max\limits_{\pi} E^{\pi} \Big[ - \sum\limits_{t=\tau-1}^T e^{-\beta sat(\sigma^F_t,\phi)} \Big], & \text{if } \Phi=G_{[0,T]} \phi
\end{cases}
\end{equation}
\textit{where $sat(\sigma^F,\phi)$ is the satisfaction function as defined in \eqref{eq:sat}.}

\subsection{Theoretical Results}

The optimal policy $\pi^*_1$ of Problem~$1$, can be related to $\pi^*_{3}$ of Problem~$3$ by the following theorem.
\begin{theorem}
Let $\Phi$ and $\phi$ be STL formula with the syntax in \eqref{STLsyntax} such that $\Phi=F_{[0,.]}\phi$ or ${\Phi=G_{[0,.]}\phi}$. Let $hrz(\Phi)=T$ .  Assume that a partial state trajectory $s_{0:\tau-1}$ is initially given where $\tau= \Big\lceil \frac{hrz(\phi)}{\Delta t} \Big\rceil+1$. For some $\beta>0$ and $\Delta t = 1$\footnote{$\Delta t =1$ is selected due to clarity in presentation, but it can be any time step.}, let $\pi^*_1$ and $\pi_{3}^*$ be the optimal policies obtained by solving Problems 1 and 3 respectively. Then,
\begin{eqnarray}\label{eq:dist}
Pr^{\pi^*_1}[s_{0:T} \models \Phi] - \frac{1}{\beta}\log (T - \tau + 2) \ \ \leq \ \ Pr^{\pi_{3}^*}[s_{0:T} \models \Phi] \ \ \leq \ \ Pr^{\pi^*_1}[s_{0:T} \models \Phi].
\end{eqnarray}
\label{theorem}
\end{theorem}
\begin{proof}
Proof 1 in Appendix
\end{proof}

For a given MDP $M$ and STL formula $\Phi$\eqref{STLsyntax}, the $F$-MDP state-space $\Sigma^F$ over which the $Q$-learning is solved has $|\Sigma^F| = |\Sigma|\times \Pi^n_1|\mathfrak{F_i}|$ number of states. $Q$-table used for $Q$-learning on $F$-MDP has $|\Sigma^F|\times |A|$ number of entries. The $\tau$-MDP [\cite{aksaray2016q}] on the other hand has $|\Sigma|^\tau$ elements in $\Sigma^\tau$ and $|\Sigma^\tau| \times |A|$ entries in its $Q$-table. For most real world problems, it is safe to assume that both the number of states ($|\Sigma|$) and horizon ($hrz(\phi) = \tau$) is more than STL sub-formulae horizons ($hrz(\phi_i) = \tau_i, \forall i \in [1,n]$) and the number of STL sub-formulae ($n$) in $\phi$ respectively. The above reasoning shows how the proposed technique has only polynomial growth of state-space with horizon $\tau$. Moreover, lesser entries in $Q$-table also results in faster learning. $Q$-learning on a smaller, compact Q-table is faster to converge to the optimal solution due to the need for the Q-learning algorithm to visit every state, action pair infinitely often, as a condition for convergence. Thus for problems with large horizon $\tau$, the proposed technique convergence to the optimal policy faster than that solved with $\tau$-MDP.

\section{Simulation Results}
\begin{wrapfigure}{r}{0.43\textwidth}
\centering
\includegraphics[width=0.35\textwidth, trim=0cm 0cm 0cm 2cm]{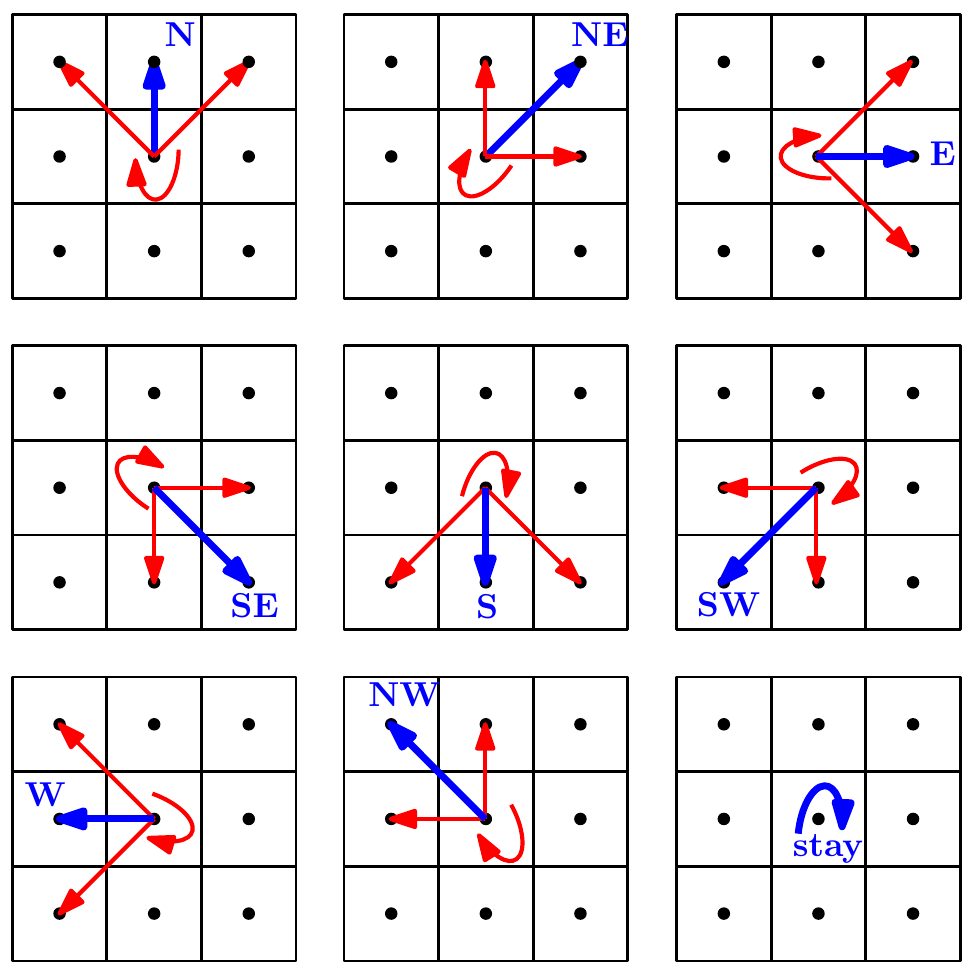}
 \caption{Motion uncertainty (red arrows) for an action (blue arrow).}
 \label{fig:motion_CS}
\end{wrapfigure} 

\label{caseStudy}
In the following case studies, we consider a single agent moving in a discretized environment. The set of motion primitives at each state is $A=\{N,NW,W,SW,S,SE,E,NE,stay\}$. We model the motion uncertainty as in Fig.~\ref{fig:motion_CS} where, for any selected feasible action in $A$, the agent follows the corresponding blue arrow with probability $0.93$ or a red arrow with probability $0.023$. Moreover, the resulting state after taking an infeasible action (i.e., the agent is next to a boundary and tries to move towards it) is the current state. All simulations were implemented in MATLAB and performed on a laptop with a quad core 2.4 GHz processor and 8.0 GB RAM.

\subsection{Case Study 1: Reachability}

\begin{wrapfigure}{r}{0.43\textwidth}
\centering
\includegraphics[width=0.35\textwidth, trim=0.1cm 0cm 0.1cm 1cm]{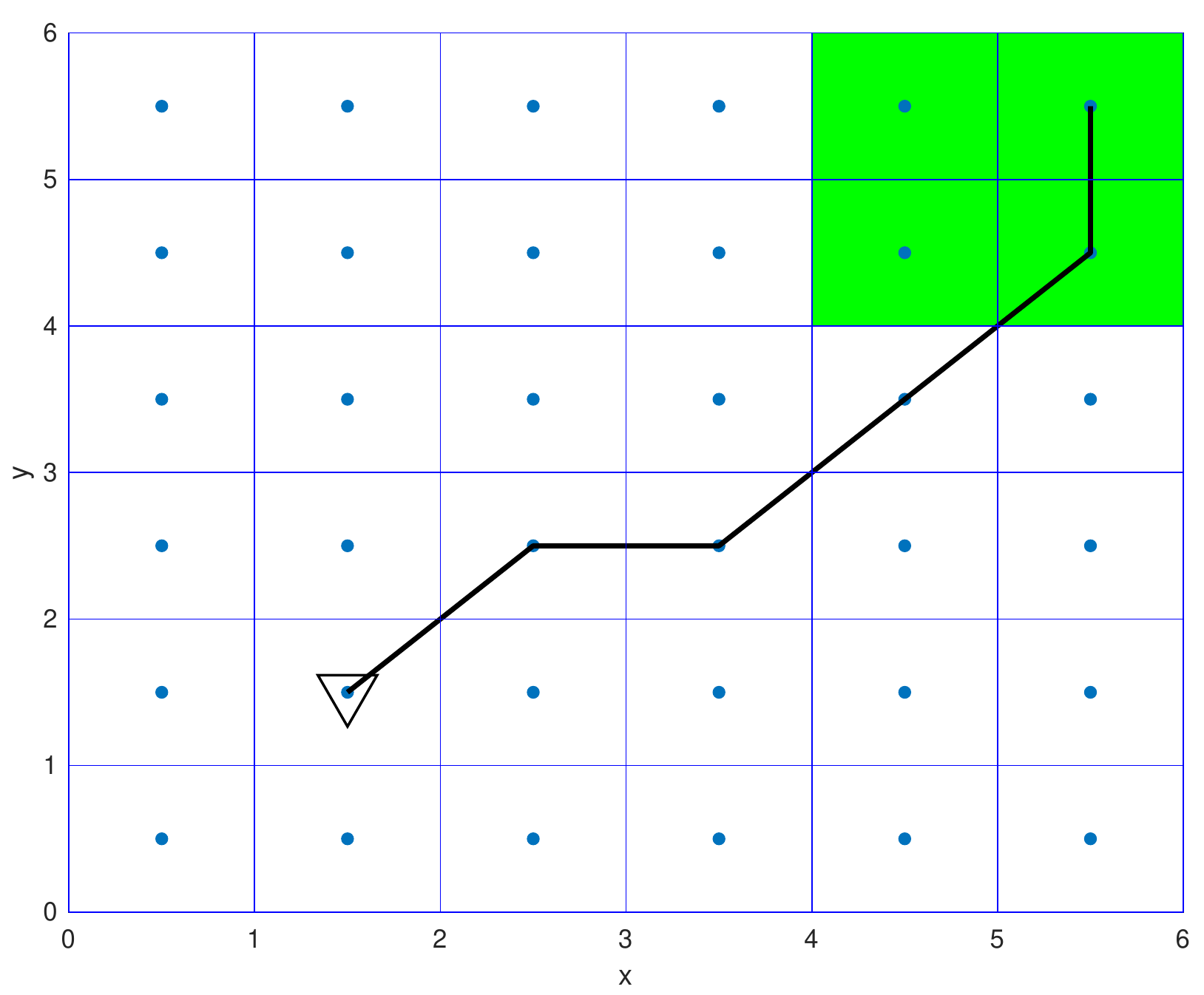}
 \caption{Sample trajectory generated by $\pi^*_{A}$ eventually reaches the desired region in green.}
 \label{fig:sampleCS1}
\end{wrapfigure}

In this case study, the initial state of the agent is $s_0=(1.5,1.5)$ as shown in Fig.~\ref{fig:sampleCS1}. We consider an STL formula defined over the environment as 
\begin{equation}
\Phi_1 = F_{[0,6]}G_{[0,1]} (x>4 \wedge y>4 ),
\label{eq:cs1}
\end{equation}
which expresses ``eventually visit the desired region within $[0,7]$. Note that $\Phi_1 = F_{[0,6]} \phi$ where ${\phi=\phi_1=G_{[0,1]}(x>4 \wedge y>4)}$ and $hrz(\phi_1)=1$. Moreover, we choose $\Delta t =1$, thus $\tau_1 =2$. The state-spaces of the system based on Fig.~\ref{fig:sampleCS1} are $|\Sigma|=36$ and $|\Sigma^F|=72$ since ${\tau_1=2}$.  

To implement the $Q$-learning algorithm, the number of episodes is chosen as $2000$ (i.e., $1\leq k \leq 2000$), and we use the parameters ${\beta=50}$, $\gamma=0.9999$, and $\alpha_k=0.95^k$. After $2000$ trainings (episodes), which took approximately $2$ minutes, the resulting policy $\pi_{3}^*$ is used to generate $500$ trajectories, which has
\[\begin{array}{cc}
{Pr^{\pi_{A}^*}[s_{0:7} \models \Phi_1]=0.996}. \\
\end{array}\]
Fig.~\ref{fig:sampleCS1} shows a sample trajectory generated by $\pi^*_3$. $\Phi_1$ is satisfied in $498/500$ sample trajectories.

\subsection{Case Study 2: Patrolling}
In the second case study, we consider an agent moving in an environment illustrated in Fig.~\ref{fig:sampleCS2} part (a). We consider an STL formula defined over the environment as 
\begin{equation}
\Phi_2 = G_{[0,12]} \big(F_{[0,h]}(region~A) \wedge F_{[0,h]}(region~B) \big),
\label{eq:cs2}
\end{equation}
where $region~A$ represents $x>1 \wedge x<2 \wedge y>3 \wedge y<4$ and $region~B$ represents $x>2 \wedge x<3 \wedge y>2 \wedge y<3$. Note that $\Phi_2$ expresses the following: ``for all $t\in[0,12]$, eventually visit $region~A$ every $[t,t+h]$ and eventually visit $region~B$ every $[t,t+h]$". Note that $\Phi_2 = G_{[0,12]} \phi$ where ${\phi=\phi_1 \wedge \phi_2}$ with $\phi_1 = F_{[0,h]}(region~A)$,  $hrz(\phi_1)=h$ and $\phi_2 = F_{[0,h]}(region~B)$, $hrz(\phi_2)=h$. $\tau=h+1$ if $\Delta t$ is $1$.   

In this case study, we chose three values for $h:=2, 4, 5$ with rest of the parameters remaining the same. The sizes of the state-spaces are $|\Sigma|=36, 36, 36$ and ${|\Sigma^F|=324, 900, 1296}$\footnote{This indicates that there are $36\times (h+1)\times (h+1)$ $F$-MDP states, $36$ system states, $(h+1)$ $f_1$ flag states and $(h+1)$ $f_2$ flag states.} for each $\tau=3, 5, 6$ respectively. To implement the $Q$-learning algorithm, the number of episodes is chosen as $10000$ (i.e., $1 \leq k \leq 10000$), with $\beta=50$, $\gamma=0.9999$, and $\alpha_k=0.95^k$. After $10000$ trainings, the resulting policies $\pi_{3}^*$ is used to generate $500$ trajectories, which leads to
\[\begin{array}{cc}
Pr^{\pi_{3}^*}[s_{0:14} \models \Phi_2]=0.6794 \ \ \ \ \  $for$ \ \ \ h=2 ,\\
Pr^{\pi_{3}^*}[s_{0:16} \models \Phi_2]=0.8237 \ \ \ \ \  $for$ \ \ \ h=4 ,\\
Pr^{\pi_{3}^*}[s_{0:17} \models \Phi_2]=0.8432 \ \ \ \ \  $for$ \ \ \ h=5 .
\end{array}\]

\begin{figure}[!htb]
\includegraphics[height=0.32\columnwidth, trim=1cm 0cm 0cm 0cm]{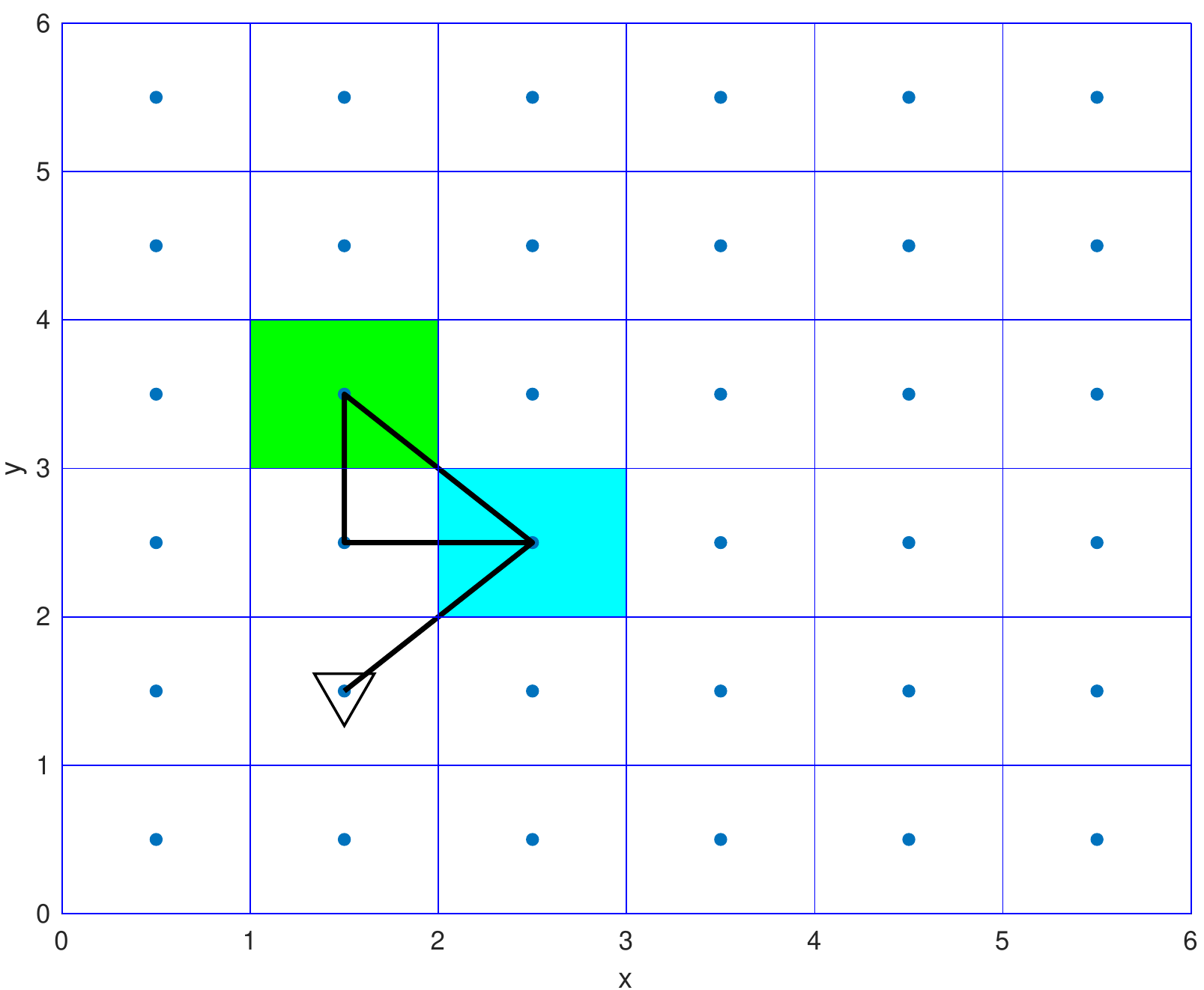}\put(-80,-10){(a)}
\includegraphics[height=0.32\columnwidth, trim=0cm 0cm 0cm 0cm]{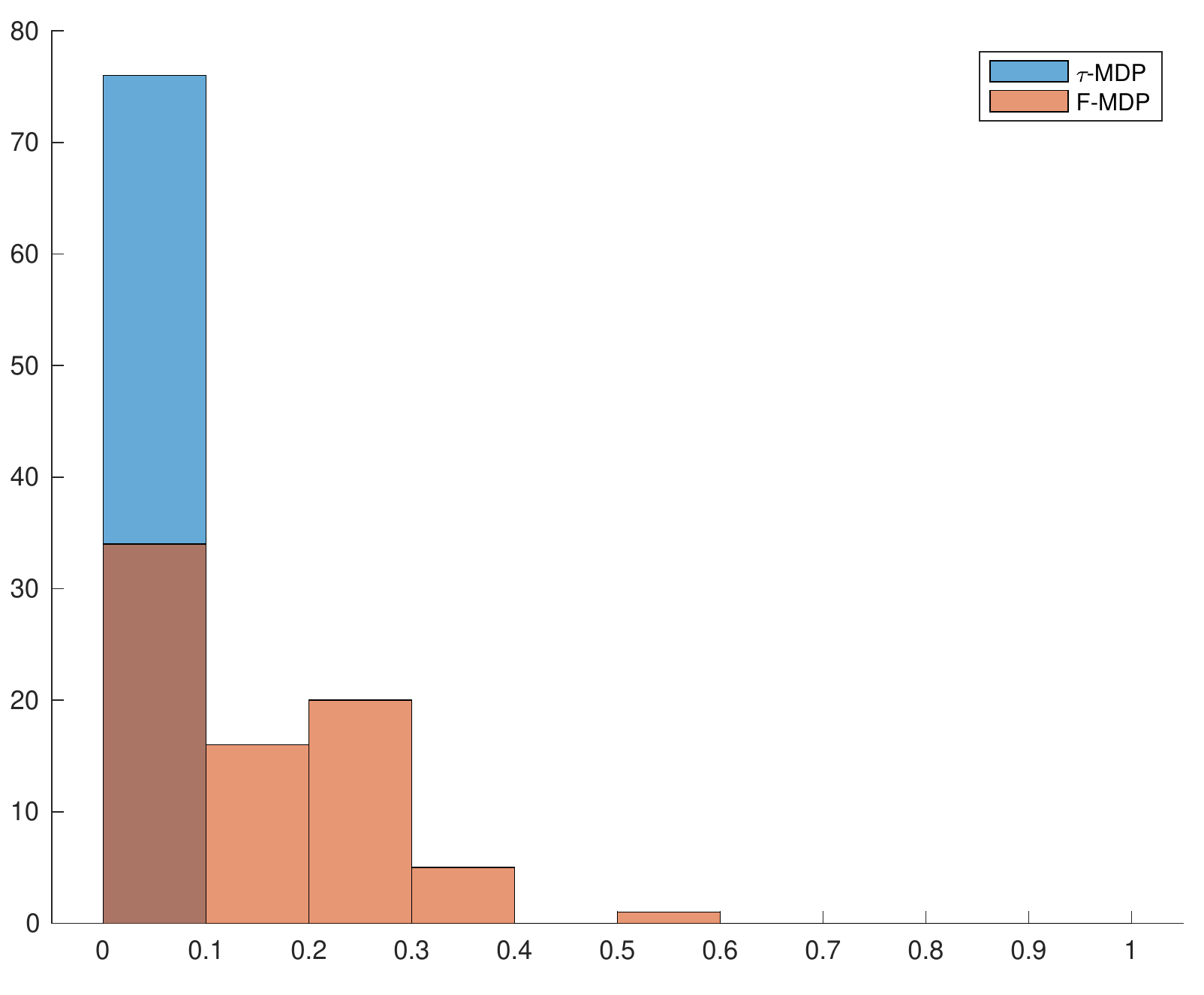}\put(-90,-10){(b)}
\centering

\includegraphics[height=0.23\columnwidth, trim=0cm 0cm 0cm 0cm]{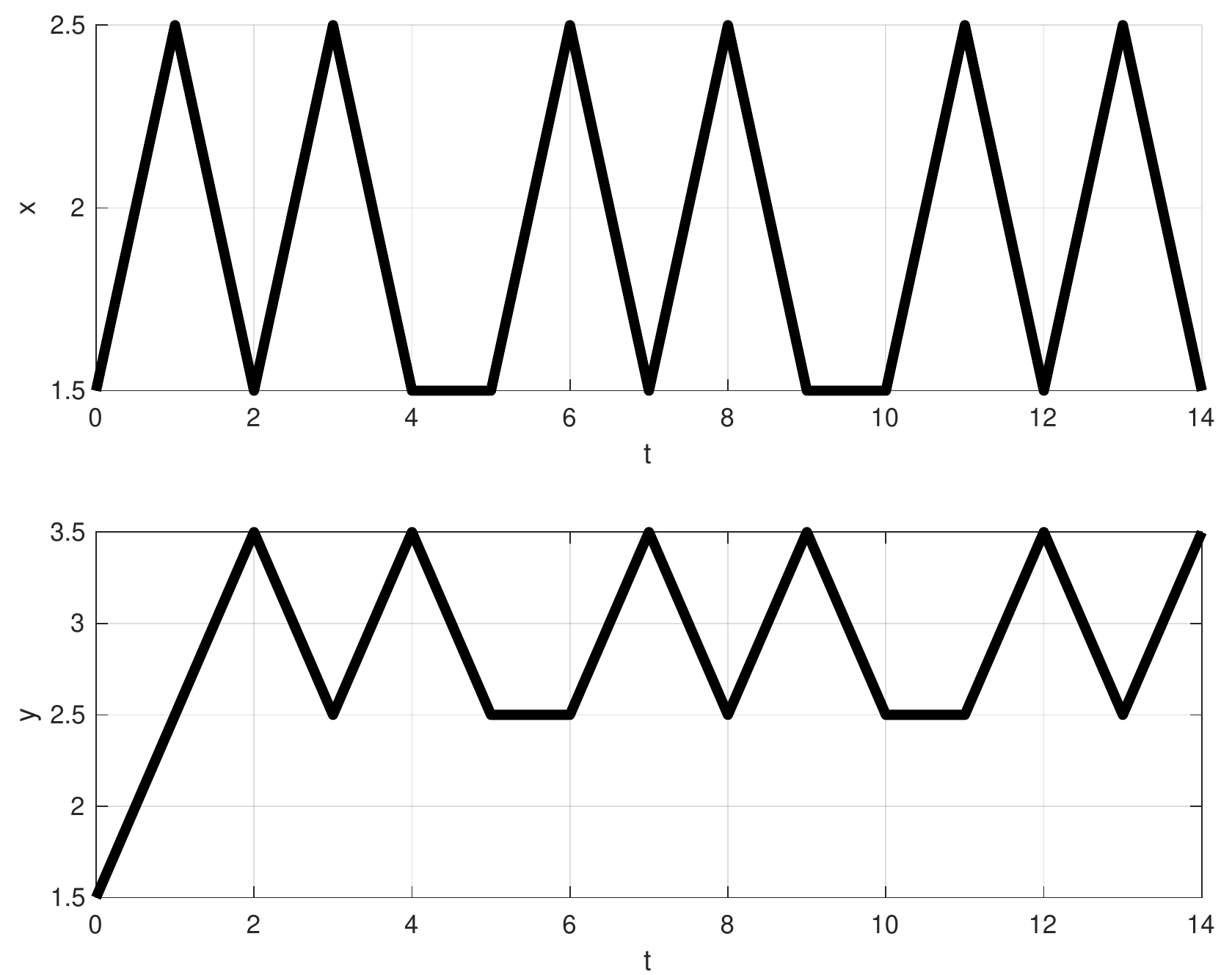}
\includegraphics[height=0.23\columnwidth, trim=0.6cm 0cm 0cm 0cm]{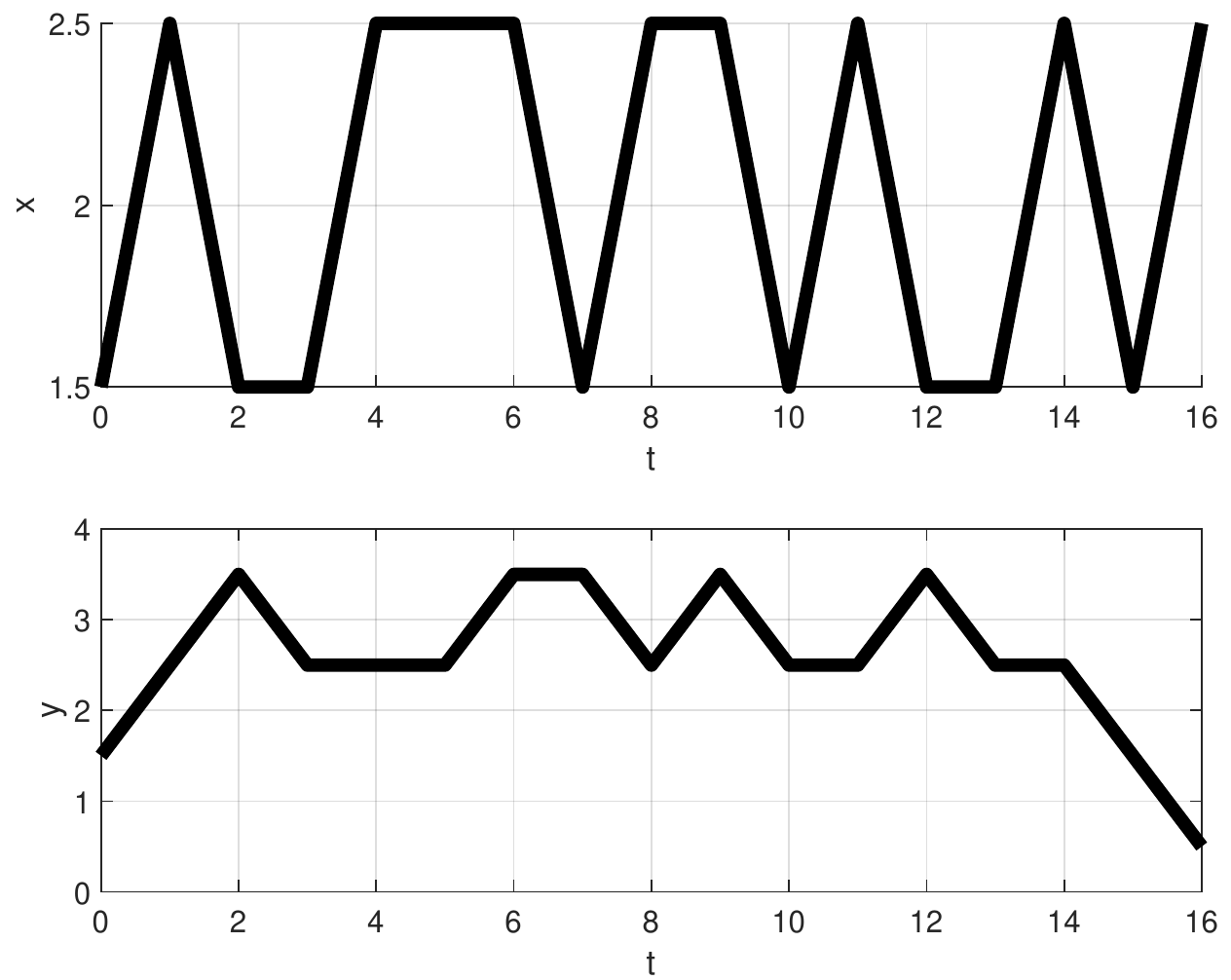}
\includegraphics[height=0.23\columnwidth, trim=0.5cm 0cm 0cm 0cm]{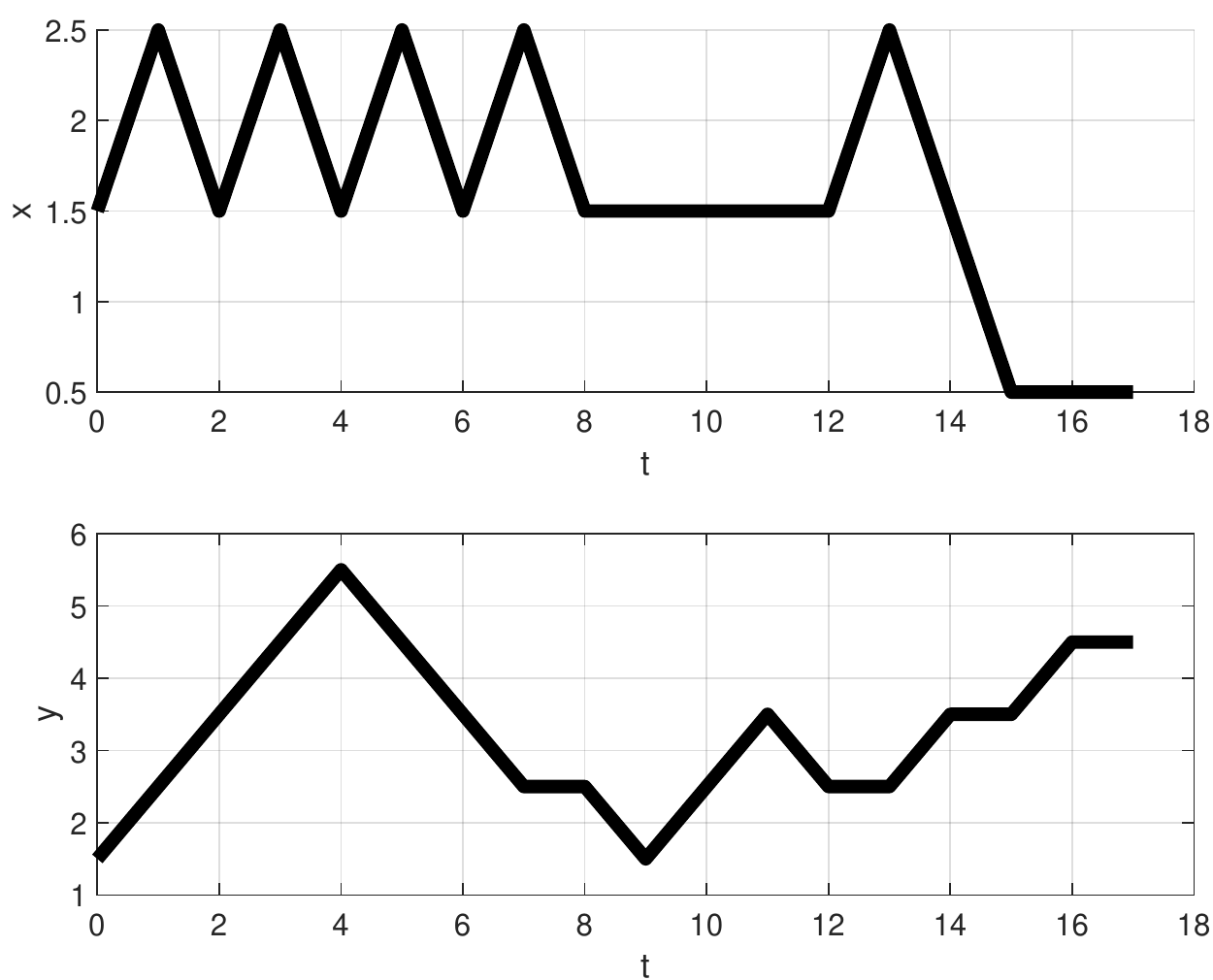}
\put(-310,-10){(c)}
\put(-189,-10){(d)}
\put(-67,-10){(e)}
 \caption{(a) Initial state and the desired regions (b) Distribution of probability of satisfaction of 75 different roll-outs with $10000$ episodes each for $\tau=3$ using $F$-MDP in orange and $\tau$-MDP in blue (c) Sample trajectory generated by $\pi^*_{A}$ for $\tau=3$ (d) $\tau=5$ (e) $\tau=6$}
 \label{fig:sampleCS2}
\end{figure}

The time and space requirements for $10000$ episode trainings on $F$-MDP and $\tau$-MDP for each of $\tau = 3,5,6$ are provided in the Table.~\ref{tab:complexity}. We can see that $F$-MDP state-space is very compact for large $\tau$. Reduced space requirement also translates into faster execution time. Fig.~\ref{fig:sampleCS2} part (b) is histogram plot of probability of satisfaction by 75 roll-outs with $10000$ episodes each for $\tau=3$ using $F$-MDP in orange and $\tau$-MDP in blue. For this problem with $\tau=3$, learning on $\tau$-MDP required 10$\times$ longer episode length to achieve similar return distribution as with the case of learning on $F$-MDP. Fig.~\ref{fig:sampleCS2} part (c), (d) and (e) shows sample trajectory generated by the optimal policy for $\tau = 3,5$ and $6$.

\begin{table}[!htb]
    \caption{Execution time (in minutes) and space requirements}
    \begin{minipage}{.5\linewidth}
      \centering
        \begin{tabular}{l|c|r|l} 
      \textbf{Technique} & \textbf{$\tau=3$} & \textbf{$\tau=5$} & \textbf{$\tau=6$}\\
      \hline
      $F$-MDP & 6 & 18 & 24\\
      $\tau$-MDP & 21 & 290 & *\textsuperscript{$\dagger$}
     \end{tabular}
     \customfootnotetext{}{$\dagger$. Matlab run time error: preferred array size exceeded}
    \end{minipage}%
    \begin{minipage}{.5\linewidth}
      \centering
        
        \begin{tabular}{l|c|r|l} 
      \textbf{Technique} & \textbf{$\tau=3$} & \textbf{$\tau=5$} & \textbf{$\tau=6$}\\
      \hline
      $F$-MDP & 2.9 $\times$ $10^3$ & 18 $\times$ $10^3$ & 24 $\times$ $10^3$ \\
      $\tau$-MDP & 1.7 $\times$ $10^4$ \textsuperscript{$\dagger$} & 1 $\times$ $10^6$ \textsuperscript{$\dagger$} & 8.1 $\times$ $10^6$ \textsuperscript{$\dagger$}
     \end{tabular}
     \customfootnotetext{}{$\ddagger$. Pruned based on feasibility of transition}
    \end{minipage} 
\label{tab:complexity}
\end{table}

\section{Conclusion}
We have proposed a model-free method to synthesize control policies that satisfy STL specifications for stochastic dynamical systems. The proposed technique remodels the system as an $F$-MDP, to capture the current system state and history. The learning objective (maximizing probability of satisfaction) is approximated and solved using $Q$-learning. We also proved that the computed optimal policy is arbitrarily close to that of the desired with a proper choice of approximation parameter. Finally, we demonstrated the tractability of the proposed technique in simulation. Future work will explore solving this problem using robustness degree metric.

\bibliography{ref}

\begin{thebibliography}{16}
\providecommand{\natexlab}[1]{#1}
\providecommand{\url}[1]{\texttt{#1}}
\expandafter\ifx\csname urlstyle\endcsname\relax
  \providecommand{\doi}[1]{doi: #1}\else
  \providecommand{\doi}{doi: \begingroup \urlstyle{rm}\Url}\fi

\bibitem[Aksaray et~al.(2015)Aksaray, Leahy, and Belta]{aksaray2015}
Derya Aksaray, Kevin Leahy, and Calin Belta.
\newblock Distributed multi-agent persistent surveillance under temporal logic
  constraints.
\newblock \emph{IFAC-PapersOnLine}, 48\penalty0 (22):\penalty0 174--179, 2015.

\bibitem[Aksaray et~al.(2016)Aksaray, Jones, Kong, Schwager, and
  Belta]{aksaray2016q}
Derya Aksaray, Austin Jones, Zhaodan Kong, Mac Schwager, and Calin Belta.
\newblock Q-learning for robust satisfaction of signal temporal logic
  specifications.
\newblock In \emph{2016 IEEE 55th Conference on Decision and Control (CDC)},
  pages 6565--6570. IEEE, 2016.

\bibitem[Brazdil et~al.(2014)Brazdil, Chatterjee, Chmelik, Forejt, Kretinsky,
  Kwiatkowska, Parker, and Ujma]{Brazdil2014}
Toma Brazdil, Krishnendu Chatterjee, Martin Chmelik, M.k, Vojtech Forejt, Jan
  Kretinsky, Marta Kwiatkowska, David Parker, and Mateusz Ujma.
\newblock Verification of markov decision processes using learning algorithms.
\newblock In Franck Cassez and Jean-Franois Raskin, editors, \emph{Automated
  Technology for Verification and Analysis}, volume 8837 of \emph{Lecture Notes
  in Computer Science}, pages 98--114. Springer International Publishing, 2014.
\newblock ISBN 978-3-319-11935-9.
\newblock \doi{10.1007/978-3-319-11936-6_8}.
\newblock URL \url{http://dx.doi.org/10.1007/978-3-319-11936-6_8}.

\bibitem[Ding et~al.(2014)Ding, Smith, Belta, and Rus]{ding2014}
Xu~Chu Ding, Stephen~L Smith, Calin Belta, and Daniela Rus.
\newblock Optimal control of markov decision processes with linear temporal
  logic constraints.
\newblock \emph{IEEE Trans. on Automatic Control}, 59\penalty0 (5):\penalty0
  1244--1257, 2014.

\bibitem[Dokhanchi et~al.(2014)Dokhanchi, Hoxha, and
  Fainekos]{dokhanchi2014line}
Adel Dokhanchi, Bardh Hoxha, and Georgios Fainekos.
\newblock On-line monitoring for temporal logic robustness.
\newblock In \emph{Runtime Verification}, pages 231--246. Springer, 2014.

\bibitem[Donz{\'e} and Maler(2010)]{donze2010}
Alexandre Donz{\'e} and Oded Maler.
\newblock \emph{Robust satisfaction of temporal logic over real-valued
  signals}.
\newblock Springer, 2010.
\newblock \doi{10.1007/978-3-642-15297-9_9}.
\newblock URL \url{http://dx.doi.org/10.1007/978-3-642-15297-9_9}.

\bibitem[Fu and Topcu(2014)]{Fu2014TLRL}
Jie Fu and Ufuk Topcu.
\newblock Probably approximately correct {MDP} learning and control with
  temporal logic constraints.
\newblock \emph{CoRR}, abs/1404.7073, 2014.
\newblock URL \url{http://arxiv.org/abs/1404.7073}.

\bibitem[Lahijanian et~al.(2015)Lahijanian, Andersson, and
  Belta]{lahijanian2015}
Morteza Lahijanian, Sean~B. Andersson, and Calin Belta.
\newblock Formal verification and synthesis for discrete-time stochastic
  systems.
\newblock \emph{IEEE Trans. on Automatic Control}, 6\penalty0 (8):\penalty0
  2031--2045, 2015.
\newblock \doi{10.1109/TAC.2015.2398883}.

\bibitem[Li and Belta(2019)]{Xiao2019_long}
Xiao Li and Calin Belta.
\newblock Temporal logic guided safe reinforcement learning using control
  barrier functions.
\newblock \emph{arXiv preprint}, 2019.

\bibitem[Li et~al.(2017)Li, Vasile, and Belta]{Xiao2017_long}
Xiao Li, Cristian-Ioan Vasile, and Calin Belta.
\newblock Reinforcement learning with temporal logic rewards.
\newblock \emph{arXiv preprint}, 2017.

\bibitem[Maler and Nickovic(2004)]{maler2004}
Oded Maler and Dejan Nickovic.
\newblock Monitoring temporal properties of continuous signals.
\newblock In Yassine Lakhnech and Sergio Yovine, editors, \emph{Formal
  Techniques, Modelling and Analysis of Timed and Fault-Tolerant Systems},
  pages 152--166. Springer, 2004.
\newblock \doi{10.1007/978-3-540-30206-3_12}.
\newblock URL \url{http://dx.doi.org/10.1007/978-3-540-30206-3_12}.

\bibitem[Sadigh et~al.(2014)Sadigh, Kim, Coogan, Sastry, and
  Seshia]{sadigh2014}
Dorsa Sadigh, Eric~S Kim, Samuel Coogan, S~Shankar Sastry, and Sanjit~A Seshia.
\newblock A learning based approach to control synthesis of markov decision
  processes for linear temporal logic specifications.
\newblock In \emph{IEEE Conf. on Decision and Control}, pages 1091--1096. IEEE,
  2014.

\bibitem[Sutton and Barto(1998)]{sutton1998}
Richard~S Sutton and Andrew~G Barto.
\newblock \emph{Reinforcement learning: An introduction}, volume~1.
\newblock MIT press Cambridge, 1998.

\bibitem[Tsitsiklis(1994)]{tsitsiklis1994}
John~N Tsitsiklis.
\newblock Asynchronous stochastic approximation and q-learning.
\newblock \emph{Machine Learning}, 16\penalty0 (3):\penalty0 185--202, 1994.

\bibitem[Watkins and Dayan(1992)]{watkins1992}
Christopher~JCH Watkins and Peter Dayan.
\newblock Q-learning.
\newblock \emph{Machine learning}, 8\penalty0 (3-4):\penalty0 279--292, 1992.

\bibitem[Xu and Topcu(2019)]{Zhe2019_long}
Zhe Xu and Ufuk Topcu.
\newblock Transfer of temporal logic formulas in reinforcement learning.
\newblock \emph{arXiv preprint}, 2019.

\end{thebibliography}

\section{Appendix}
\textbf{\em Log-Sum-Exp \em Approximation:}
Given $n$ data points $x_1,x_2,...x_n$ with the objective of finding maximum/minimum of the sequence.  The maximum/minimum can be approximated as follows:
\begin{equation}
\begin{aligned}
\max (x_1,\dots,x_n) \sim \frac{1}{\beta} \log \sum_{i=1}^n e^{\beta x_i},\\
\min (x_1,\dots,x_n) \sim -\frac{1}{\beta} \log \sum_{i=1}^n e^{-\beta x_i},
\end{aligned}
\label{eq:softmax}
\end{equation}
where $\beta>0$ is a constant. The error bounds of this approximation is as follows:
\begin{equation}
\begin{aligned}
\max (x_1,\dots,x_n) \leq \frac{1}{\beta}\log \sum_{i=1}^n e^{\beta x_i} \leq \max (x_1,\dots,x_n) + \frac{1}{\beta}\log n,\\
\min (x_1,\dots,x_n) - \frac{1}{\beta}\log n \leq -\frac{1}{\beta}\log \sum_{i=1}^n e^{-\beta x_i} \leq \min (x_1,\dots,x_n).
\end{aligned}
\label{eq:softmaxgap}
\end{equation}
With arbitrarily large $\beta$, the approximation becomes the exact solution.\\\\\\
\textbf{Proof 1:}
Using \em log-sum-exp \em approximation, we get
\begin{equation}
\begin{aligned}
g(\pmb{\sigma})=
\begin{cases}
    \max\limits_{t \in [\tau-1,T]} \sigma_{t-\tau+1:t} \models \phi \\
    \min\limits_{t \in [\tau-1,T]} \sigma_{t-\tau+1:t} \models \phi 
\end{cases} \approx \quad \hat{g}(\pmb{\sigma}) = 
\begin{cases}
    \frac{1}{\beta}\log \sum\limits_{t = \tau-1}^T e^{\beta (\sigma_{t-\tau+1:t} \models \phi)} & \text{if $\Phi = F_{[0,T]}\phi $}\\
    -\frac{1}{\beta}\log \sum\limits_{t = \tau-1}^T e^{-\beta (\sigma_{t-\tau+1:t} \models \phi)} & \text{if $\Phi = G_{[0,T]}\phi $}
\end{cases}
\end{aligned}
\end{equation}
Since $1/\beta$ is a constant and log(.) a strictly monotonous function. $\widehat{g}(\pmb{\sigma})$ with log(.) and $1/\beta$ dropped from $\hat{g}(\pmb{\sigma})$ can substitute $\hat{g}(\pmb{\sigma})$ in equation $\pi^* = \arg\max\limits_{\pi}E^{\pi}[\hat{g}(\pmb{\sigma}))]$ to get equation \eqref{pr:1a}.

Using \eqref{eq:softmaxgap} we can obtain the following inequality
\begin{align*}
E^{\pi} [\max\limits_{t \in [\tau-1,T]} (\sigma_{t-\tau+1:t} \models \phi)] & \leq E^{\pi} \Big[ \sum\limits_{t=\tau-1}^T e^{\beta sat(\sigma^F_t,\phi)} \Big] \\ & \leq E^{\pi} [\max\limits_{t \in [\tau-1,T]} (\sigma_{t-\tau+1:t} \models \phi)] + \frac{1}{\beta}\log (T-\tau+2)\\
 \end{align*}
\begin{align}
E^{\pi} [\min\limits_{t \in [\tau-1,T]} (\sigma_{t-\tau+1:t} \models \phi)] - \frac{1}{\beta}\log (T-\tau+2) & \leq E^{\pi} \Big[- \sum\limits_{t=\tau-1}^T e^{-\beta sat(\sigma^F_t,\phi)} \Big] \nonumber \\ & \leq E^{\pi} [\min\limits_{t \in [\tau-1,T]} (\sigma_{t-\tau+1:t} \models \phi)]
\label{eq:ineq}
\end{align}
In \eqref{eq:ineq} we use policy $\pi = \pi^*_1,\pi^*_3$, which along with the following inequalities 

\begin{align*}
E^{\pi^*_3} [\max\limits_{t \in [\tau-1,T]} (\sigma_{t-\tau+1:t} \models \phi)] & \leq E^{\pi^*_1} [\max\limits_{t \in [\tau-1,T]} (\sigma_{t-\tau+1:t} \models \phi)] \\ 
E^{\pi^*_3} [\min\limits_{t \in [\tau-1,T]} (\sigma_{t-\tau+1:t} \models \phi)] & \leq E^{\pi^*_1} [\min\limits_{t \in [\tau-1,T]} (\sigma_{t-\tau+1:t} \models \phi)] \\ 
E^{\pi^*_1} \Big[ \sum\limits_{t=\tau-1}^T e^{\beta sat(\sigma^F_t,\phi)} \Big] & \leq E^{\pi^*_3} \Big[ \sum\limits_{t=\tau-1}^T e^{\beta sat(\sigma^F_t,\phi)} \Big] 
\end{align*}
\begin{align}
E^{\pi^*_1} \Big[ -\sum\limits_{t=\tau-1}^T e^{-\beta sat(\sigma^F_t,\phi)} \Big] & \leq E^{\pi^*_3} \Big[ -\sum\limits_{t=\tau-1}^T e^{-\beta sat(\sigma^F_t,\phi)} \Big]
\label{eq:relate}
\end{align}

can be solved together to obtain \eqref{eq:dist}.

\end{document}